\documentclass[pdflatex,sn-basic,Numbered]{sn-jnl}

\usepackage{amsmath}
\usepackage{amssymb}
\usepackage{mathtools}
\usepackage{bm}
\usepackage{amsthm}
\theoremstyle{thmstyleone}
\newtheorem{theorem}{Theorem}[section]

\usepackage{booktabs}
\usepackage{makecell}
\usepackage{tabularx}
\usepackage{longtable}
\usepackage{pdflscape}
\usepackage{array}
\usepackage{graphicx}

\usepackage{capt-of}
\usepackage{xcolor}
\usepackage{colortbl}
\definecolor{rowgray}{RGB}{246,246,246}
\definecolor{bestcell}{RGB}{252,238,216}
\usepackage{microtype}

\usepackage{placeins}
\usepackage{needspace}
\usepackage{url}
\hypersetup{hidelinks}
\graphicspath{{./}}

\makeatletter
\newcommand{\projectleadmark}{\g@addto@macro\artauthors{\textsuperscript{$\dagger$}}}
\newcommand{\correspondingmark}{\g@addto@macro\artauthors{\textsuperscript{$\ddagger$}}}
\newcommand{\authorroleline}{\g@addto@macro\auaddress{\textsuperscript{$\dagger$}Project lead.\quad\textsuperscript{$\ddagger$}Corresponding author: \href{mailto:2024002850@mails.cust.edu.cn}{2024002850@mails.cust.edu.cn}\par}}
\makeatother

\begin{document}
\title{R1A-PC: Physics-Guided Electromagnetic Inversion of Three-Dimensional Human Point Clouds in Complex Static Environments}
\author[1]{\textbf{Xudong Yuan}}
\author[1]{\textbf{Ruyun Xu}}
\author[1]{\textbf{Jingtai Yang}}\projectleadmark
\author[1]{\textbf{Xianzheng Sun}}\correspondingmark
\affil[1]{\orgdiv{School of Mathematics and Statistics}, \orgname{Changchun University of Science and Technology}, \city{Changchun}, \country{China}}
\authorroleline
\abstract{Recovering three-dimensional human geometry from electromagnetic measurements in a complex static environment is difficult because strong multipath responses from walls, floors, and other objects obscure the weak target perturbation. We propose R1A-PC, a physics-guided method that reconstructs a 2048-point human cloud from paired complex fields measured with and without the target. Complex background subtraction emphasizes target-induced amplitude and phase changes, while the background field remains available as an environmental condition. A frequency-balanced discrete Born adjoint produces a three-dimensional spatial knowledge map. At each of two bounded deformation stages, the decoder combines complex measurement features, background features, and multiscale physical features queried at the current point coordinates; the second stage queries again after the first coordinate update. We analyze the residual of paired subtraction, the weighted normal-operator structure of the raw adjoint, and the feasible set of the predicted cloud. In a held-out background generated by full-wave simulation under a fixed acquisition geometry, R1A-PC obtains a squared Chamfer distance of 0.001434~m$^2$ and an F-score of 0.963080 at 0.05~m. Compared with TopNet, the Chamfer distance decreases by 70.28\%. Removing physical guidance or background subtraction increases the Chamfer distance by 242.06\% or 241.18\%, respectively. Experiments across background layouts and poses support the complementary roles of paired subtraction and position-dependent adjoint features.}
\maketitle

\section{Introduction}\label{sec:introduction}

Contactless three-dimensional human sensing can provide geometric information for indoor safety, intelligent spaces, and activity understanding under occlusion. A point cloud represents position, body shape, and limb pose in a common coordinate system. Electromagnetic measurements do not depend on visible-light texture and can record complex responses related to propagation and target structure in low light or partial occlusion \cite{chen2023meshfree,bi2024diffusionemis}. Recovering geometry from limited transmitters, receivers, and frequencies requires a map from propagation-mixed fields to an unordered set of surface points. Analytical inversion supplies propagation structure, whereas learned methods supply a prior on human geometry \cite{colton2019inverse,arridge2019solving}.

Two coupled difficulties arise in complex static scenes. First, strong multipath responses from walls, floors, and fixed objects may dominate the target increment. A model trained directly on target-present fields can learn shortcuts based on the background and then misplace or distort the body in an unseen scene. Subtracting a paired background response suppresses the shared static component but does not remove target--background interactions. Second, a finite aperture yields ambiguous spatial support, and complex measurements have no pointwise correspondence with an unordered output cloud \cite{bornwolf1999principles,colton2019inverse,engl1996regularization}. Point-set generators and electromagnetic point-cloud inversion demonstrate the feasibility of fixed-cardinality reconstruction \cite{tchapmi2019topnet,yuan2018pcn,chen2023meshfree,bi2024diffusionemis}, but a global measurement code gives each generated point little location-specific evidence. Adjoint backprojection provides a location cue but is blurred by the point-spread function and sidelobes.

We address the first difficulty with paired target-present and background fields acquired in the same environment. Their complex difference emphasizes target-induced changes in both amplitude and phase, while an encoding of the background retains information about the propagation environment. For the second difficulty, the conjugate adjoint of a discrete Born operator maps the measurements to a common domain-of-interest (DOI) grid. A multiscale volume encoder preserves this spatial evidence, which is sampled at the current coordinates of each point. After the first deformation, the second stage samples the physical features again at the updated coordinates.

The contributions are threefold. First, paired complex subtraction and background conditioning are combined, and a residual bound identifies the effect of mismatch between paired measurements. Second, a frequency-balanced Born adjoint is connected to a two-stage, bounded, pointwise deformation by multiscale queries at current point positions. A weighted normal-operator analysis and a feasible-set bound clarify the spatial information and coordinate constraints. Third, full-wave simulations with background-disjoint splits evaluate the principal test set, component ablations, transfer across environments, and seven additional pose groups. On env\_004, R1A-PC reduces squared Chamfer distance (CD-L2) by 70.28\% relative to TopNet; removing physical guidance or background subtraction increases CD-L2 by 242.06\% or 241.18\%.

Figure~\ref{fig:measurement-configuration} depicts the acquisition geometry. Section~\ref{sec:related-work} reviews related work, Section~\ref{sec:problem-formulation} defines the inverse task, Section~\ref{sec:method} presents R1A-PC, Section~\ref{sec:theory} analyzes its components, and Section~\ref{sec:experiments} reports the experiments.

\begin{figure}[!tb]
\centering
\ifdim\columnwidth<\textwidth
\includegraphics[width=\columnwidth]{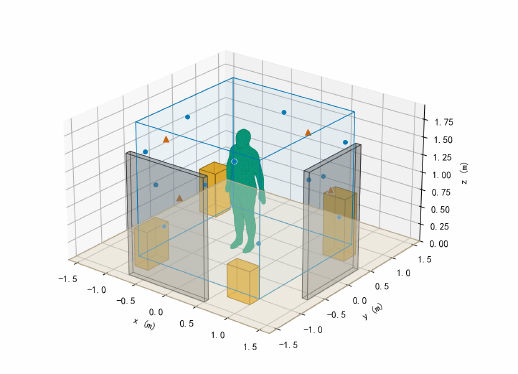}
\else
\includegraphics[width=0.53\textwidth]{Fig1.pdf}
\fi
\caption{Three-dimensional electromagnetic acquisition geometry in a complex static environment. The env\_004 example shows the DOI, four transmitters, twelve receiver locations, walls, static objects, and the human mesh; coordinates are in metres}
\label{fig:measurement-configuration}
\end{figure}

\section{Related Work}\label{sec:related-work}

\subsection{Learning-Based Electromagnetic Inverse Scattering}\label{sec:related-em}

Electromagnetic inverse scattering balances a forward propagation model, measurement error, and prior regularization. Born linearization and adjoint backprojection can rapidly locate scattering support, but multiple scattering, limited aperture, and the point-spread function of the normal operator limit quantitative recovery. Learned inversion absorbs part of the regularization or iteration into training, through direct regression, learned unrolling, or operator-driven reconstruction. Generalization depends on the training distribution, regularization, and forward-model error \cite{arridge2019solving}.

Recent electromagnetic methods incorporate physical structure through residual learning, enhanced imaging, multichannel fusion, self-supervised unrolling, conditional generation, and time-domain models \cite{guo2023residual,desmal2024deeplearning,chen2023earlyfusion,liu2024selfsupervised,yu2024conditional,yang2024timedomain}. Physics-induced networks, multi-imaging-physics unrolling, prior-informed models, and attention-based inversion further show the value of explicit operators and physical features \cite{liu2024physicsinduced,kofler2024miphduo,wei2024prior,chen2023general,li2025physicsassisted,zhang2025attention}. Most outputs, however, are regular-grid images. Human point clouds in cluttered rooms require both an unordered coordinate decoder and a way to distinguish target perturbations from scene-dependent propagation.

Paired background subtraction provides control at the measurement level. It attenuates the response shared by two acquisitions in the same static scene before the residual complex field reaches the learner. Its effectiveness depends on consistent acquisition: drift and independent noise remain in the difference. We therefore retain the background field as a separate condition, analyze the residual explicitly, and use environment-level data separation to test the effect of changed propagation conditions.

\subsection{Physics-Guided Three-Dimensional Point-Cloud Recovery}\label{sec:related-pointcloud}

Physics-informed learning introduces equations or discrete operators into the input, architecture, or objective. Equation residuals and hybrid models can improve data efficiency and extrapolation \cite{raissi2019physics,karniadakis2021physics}; work on knowledge calibration, stable inverse maps, and finite-volume networks also highlights the roles of scaling, discretization error, and the injection point \cite{zhuang2025stable,chen2023knowledge,mei2024unified}. Here, the discrete adjoint offers spatial evidence but cannot directly produce a human surface: its output lives on a contrast grid, whereas the desired cloud consists of continuous coordinates.

A global physical-consistency loss does not necessarily supply a direction for each output point. Compressing a backprojection into one latent vector weakens its coordinate information, while treating high-valued adjoint voxels as surface points inherits aperture blur and sidelobes. Retaining a multiscale spatial field and querying it continuously at each current point position connects the grid representation to coordinate updates.

Unordered point-set methods use symmetric aggregation, folding, set distances, or tree generators to model surfaces with a fixed number of points \cite{qi2017pointnet,yang2018foldingnet,fan2017pointset,tchapmi2019topnet}. Coarse-to-fine completion appears in PCN, PoinTr, and SnowflakeNet \cite{yuan2018pcn,yu2021pointr,xiang2021snowflake}. Subsequent work strengthens local relations through geometry-aware queries, anchors, hyperbolic distances, and view conditions \cite{yu2023adapointr,zhang2023anchorformer,lin2023hyperbolic,fu2023vapcnet}, or reduces dependence on complete inputs through reprojection and latent-shape learning \cite{cui2023p2c,wu2024fsc,chen2024walkformer,li2024latent}. In electromagnetic inversion, the input is a cross-modal field rather than a partial set of visible target points.

Prior electromagnetic point-cloud inversion shows that unordered coordinates can avoid the representational burden of a large output grid; diffusion-based generation can model more complex geometry \cite{chen2023meshfree,bi2024diffusionemis}. These studies mainly consider free space or simpler backgrounds. R1A-PC adds paired background conditioning and updates the adjoint evidence used by each point as its coordinates change.

\section{Problem Formulation}\label{sec:problem-formulation}

Let $(e,t,f,r,c)$ index the static environment, transmitter, frequency, receiver, and electric-field component. The target-present and background responses, acquired under the same geometry and in the same environment, are
\begin{equation}
\mathbf E_e^{\mathrm{obj}},\mathbf E_e^{\mathrm{bg}}
\in\mathbb C^{T\times F\times R\times C},
\qquad (T,F,R,C)=(4,3,12,3),
\label{eq:paired-fields}
\end{equation}
The frequencies are 300, 450, and 600~MHz. The real and imaginary parts of all three field components are retained. The target-related observation is the elementwise complex difference
\begin{equation}
\Delta\mathbf E_e=\mathbf E_e^{\mathrm{obj}}-\mathbf E_e^{\mathrm{bg}}.
\label{eq:complex-difference}
\end{equation}

Because the paired responses share the static scene and acquisition geometry, Eq.~\eqref{eq:complex-difference} suppresses their common scattering contribution while preserving target-induced amplitude and phase changes. Multipath propagation, occlusion, and target--background coupling remain. The inverse map therefore receives both $\Delta\mathbf E_e$ and $\mathbf E_e^{\mathrm{bg}}$: the former emphasizes the target increment, and the latter conditions the interpretation of residual scene effects. Subtracting the complex components rather than magnitudes preserves relative phase information associated with delay and interference.

The domain of interest is
\begin{equation}
\Omega=[-1,1]\times[-1,1]\times[0,1.9]\ \mathrm m^3.
\label{eq:doi}
\end{equation}
A common DOI sampling grid defines the spatial knowledge representation; it is distinct from the finite-element degrees of freedom used by the full-wave solver. The model predicts 2048 metric coordinates
\begin{equation}
\widehat{\mathcal P}=\{\widehat{\mathbf p}_i\}_{i=1}^{2048}
\in\mathbb R^{2048\times3},
\qquad \widehat{\mathbf p}_i\in\Omega,
\label{eq:pointcloud-output}
\end{equation}
and learns the conditional inverse map
\begin{equation}
\widehat{\mathcal P}=\mathcal F_\theta
\left(\Delta\mathbf E_e,\mathbf E_e^{\mathrm{bg}},\bar S_e\right),
\label{eq:inverse-map}
\end{equation}
The input is a complex tensor indexed by transmitter, frequency, receiver, and component, whereas the output is an unordered point set. Limited acquisition channels leave the spatial distribution underdetermined; multiple frequencies, the adjoint map, a human-shape prior, and the environmental condition supply complementary constraints. The reconstruction must recover position, limb orientation, and surface coverage within $\Omega$ without depending on an arbitrary ordering of target points.

Here $\bar S_e$ is a three-dimensional physics-derived knowledge map constructed from the complex difference and a discrete adjoint. The formal output consists only of unordered point coordinates.

\section{R1A-PC Method}\label{sec:method}

\begin{figure*}[!tp]
\centering
\includegraphics[width=0.94\textwidth]{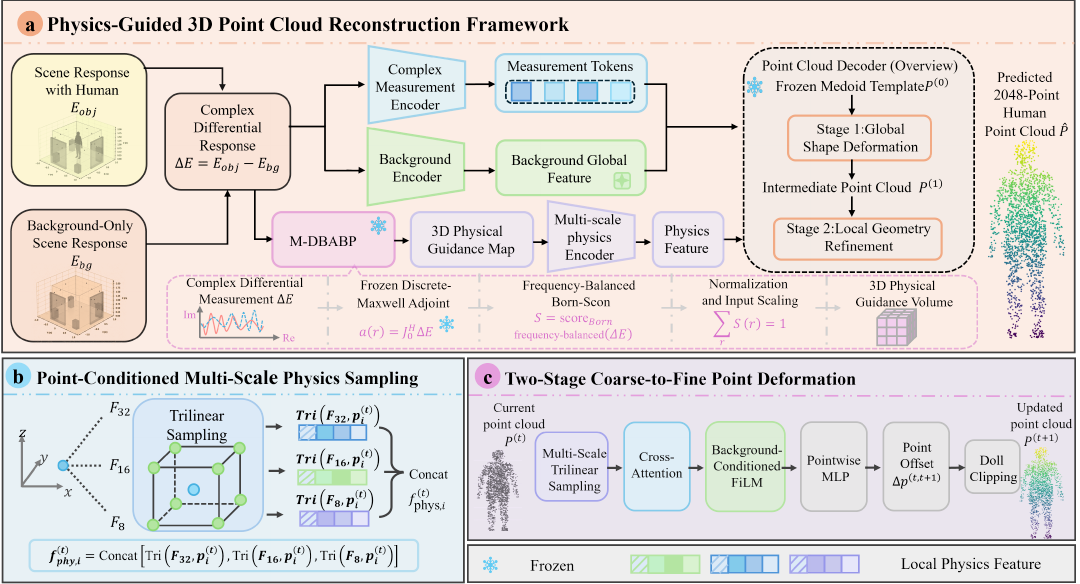}
\caption{R1A-PC architecture. Paired complex responses yield measurement, background, and adjoint spatial conditions. Multiscale queries at current point locations drive two bounded deformation stages to output a 2048-point human cloud}
\label{fig:architecture}
\end{figure*}

\subsection{Overview}\label{sec:method-overview}

R1A-PC comprises complex measurement encoding, background conditioning, a frequency-balanced adjoint, multiscale volumetric encoding, and two-stage point-cloud deformation (Fig.~\ref{fig:architecture}). The difference $\Delta\mathbf E_e$ is encoded as a sequence that retains relations among transmitters, frequencies, and receivers, and is also backprojected by the discrete Born adjoint to form spatial evidence. A separate encoder maps $\mathbf E_e^{\mathrm{bg}}$ to a global environmental condition. Starting from a 2048-point template determined from training data, the decoder queries physical features at current point coordinates and combines them with the measurement sequence and background condition to update each point.

The three paths carry complementary information. The measurement sequence preserves complex target-induced changes across acquisition channels; the adjoint path places the same perturbation in the DOI; and the background path summarizes scene-dependent propagation. Measurement and background representations are shared by the points of one sample, while each point reads different local evidence from the shared physical volume. The first stage adjusts position and coarse shape; the second requeries the volume at updated coordinates and refines the surface.

\subsection{Complex Difference and Background Conditioning}\label{sec:complex-encoding}

For a fixed $(t,f,r)$, the three Cartesian field components form a token of paired real and imaginary values:
\begin{equation}
\begin{aligned}
\boldsymbol\tau_{tfr}
&=\bigl[
\operatorname{Re}\Delta E_x,\operatorname{Im}\Delta E_x,
\operatorname{Re}\Delta E_y,\operatorname{Im}\Delta E_y,\\
&\qquad
\operatorname{Re}\Delta E_z,\operatorname{Im}\Delta E_z
\bigr]_{tfr}\in\mathbb R^6.
\end{aligned}
\label{eq:complex-token}
\end{equation}
This Cartesian representation is continuous even near zero amplitude and retains the information needed to represent phase. Encoding all tokens gives a measurement sequence
\begin{equation}
\mathbf Z_m=\operatorname{Enc}_m
\left(\{\boldsymbol\tau_{tfr}\}_{t,f,r}\right),
\label{eq:measurement-encoder}
\end{equation}
which each template point reads through cross-attention. The paired background response enters a separate encoder in the same real--imaginary organization and is averaged to a sample-level condition $\mathbf z_b$. Keeping the target increment and the environmental response in separate branches prevents the difference alone from having to explain residual propagation effects.

The measurement sequence is not reduced to a single vector before decoding. A point forms a query from its current state and can attend to acquisition channels that are informative for its location. The background condition modulates the fused point features, accounting for propagation effects shared across the body; local differences are supplied by the adjoint volume, current coordinates, and measurement attention. Paired subtraction removes the common static response, but background drift and target--environment interactions can remain, motivating this separate conditioning branch.

\subsection{Frequency-Balanced Adjoint Knowledge}\label{sec:adjoint-knowledge}

Flatten $\Delta\mathbf E_e$ to $\Delta\mathbf y_e$. The discrete Born operator $\mathbf A_{e,t,f}$ maps complex contrast on a common DOI grid to the receivers for transmitter $t$ and frequency $f$; its conjugate adjoint maps complex residuals back to that grid. To balance the overall amplitudes of the frequencies, define the sample-specific scale and raw adjoint field by
\begin{equation}
d_{e,f}=\left\|\Delta\mathbf y_{e,:,f,:,:}\right\|_2+10^{-12},
\qquad
\mathbf g_{e,t,f}=\mathbf A_{e,t,f}^{H}
\left(\frac{\Delta\mathbf y_{e,t,f}}{d_{e,f}}\right).
\label{eq:raw-adjoint}
\end{equation}
Transmitters at the same frequency share this scale, preserving their relative energy. After taking magnitudes, the raw complex adjoint fields are aggregated across transmitters and frequencies:
\begin{equation}
\begin{aligned}
S_e(\mathbf r)
&=\sum_{t=1}^{4}\sum_{f=1}^{3}|g_{e,t,f}(\mathbf r)|,\\
\bar S_e(\mathbf r)
&=\frac{S_e(\mathbf r)}{\sum_{\mathbf q\in\Omega_h}S_e(\mathbf q)},\\
V_e&=32768\,\bar S_e.
\end{aligned}
\label{eq:knowledge-map}
\end{equation}
Taking magnitudes before aggregation avoids cancellation among illuminations with different phases. Spatial normalization emphasizes the relative support within the DOI; the full complex relation remains available in the measurement branch. Because scaling, magnitudes, and normalization make $V_e$ data-dependent and nonlinear, the normal-operator interpretation in Section~\ref{sec:theory-normal} applies specifically to the raw adjoint before these operations.

The adjoint maps measurement sensitivity to a spatial coordinate system using the known propagation structure. Finite aperture, Born-model error, and background mismatch still yield broadened peaks and sidelobes. Frequency balancing uses one norm per frequency over all transmitters, receivers, and field components, so the relative strengths of illuminations at a given frequency are preserved. The factor 32768 changes the scale, but not the ranking, of grid locations. The DOI grid is a common sampling representation, not the full-wave solver mesh; its values are support evidence rather than direct surface points.

\subsection{Multiscale Volumetric Encoding and Current-Point Queries}\label{sec:multiscale-fusion}

A three-dimensional encoder transforms $V_e$ into features at three spatial scales:
\begin{equation}
\mathbf F_{32}\in\mathbb R^{16\times32^3},\quad
\mathbf F_{16}\in\mathbb R^{32\times16^3},\quad
\mathbf F_{8}\in\mathbb R^{64\times8^3}.
\label{eq:physics-pyramid}
\end{equation}
The fine scale preserves differences near adjoint peaks; coarser scales provide a wider receptive field for body structure. For the current point $\mathbf p_i^{(s)}$ at stage $s$, its local physical feature is
\begin{equation}
\mathbf f_{\mathrm{phys},i}^{(s)}=
\operatorname{Concat}\!\left[
\operatorname{Tri}(\mathbf F_{32},\mathbf p_i^{(s)}),
\operatorname{Tri}(\mathbf F_{16},\mathbf p_i^{(s)}),
\operatorname{Tri}(\mathbf F_{8},\mathbf p_i^{(s)})
\right].
\label{eq:trilinear-fusion}
\end{equation}
Here $\operatorname{Tri}$ denotes trilinear interpolation in the metric DOI coordinates. It gives a continuous pointwise feature even between voxel centers. After the first deformation, the second stage queries all three scales at the updated coordinates. Coarse context helps distinguish broad target support from isolated sidelobes, while fine features help localize the surface. The volume is computed once for the sample and reused for all 2048 queries.

\subsection{Two-Stage Bounded Point-Cloud Deformation}\label{sec:bounded-deformation}

Each point query combines a template-index feature, the current coordinate, and $\mathbf f_{\mathrm{phys},i}^{(s)}$, and attends to $\mathbf Z_m$. The background feature linearly modulates the fused point representation:
\begin{equation}
\begin{aligned}
\widetilde{\mathbf h}_i^{(s)}
&=\mathbf h_i^{(s)}\odot
\left(\mathbf 1+\tanh\boldsymbol\gamma_i^{(s)}\right)
+\boldsymbol\beta_i^{(s)},\\
(\boldsymbol\gamma_i^{(s)},\boldsymbol\beta_i^{(s)})
&=\operatorname{FiLM}_s(\mathbf z_b).
\end{aligned}
\label{eq:film}
\end{equation}
Measurement attention selects relevant complex acquisition channels; background modulation adjusts propagation-dependent bias. The stagewise displacement and coordinate update are
\begin{equation}
\begin{aligned}
\Delta\mathbf p_i^{(s)}
&=a_s\tanh\!\left(\operatorname{MLP}_s
(\widetilde{\mathbf h}_i^{(s)})\right),\\
\mathbf p_i^{(s)}
&=\Pi_\Omega\!\left(\mathbf p_i^{(s-1)}+
\Delta\mathbf p_i^{(s)}\right),\\
a_1&=0.25\ \mathrm m,\qquad a_2=0.10\ \mathrm m.
\end{aligned}
\label{eq:stage-update}
\end{equation}
The first stage allows a larger change in pose, position, and outline; the second refines local shape with a smaller step. At each stage, the box projection $\Pi_\Omega$ keeps the point within the DOI, so the next physical query is well-defined there. The final reconstruction is $\widehat{\mathcal P}=\mathcal P^{(2)}$.

Template indices provide persistent point identities across the two updates, but no explicit topological edges are imposed. The learned set distance supervises spatial coverage. In Eq.~\eqref{eq:film}, the multiplicative term changes feature gain, the additive term compensates scene-dependent bias, and $\tanh$ bounds the gain change. The displacement bounds control the per-stage search range; body-part directions and local density are learned from the fused evidence and training examples.

Figure~\ref{fig:deformation-stages} compares first-stage predictions (P1), final predictions (P2), and ground truth for four fixed env\_003 examples. The panels use the same coordinate limits, orthographic camera, point size, and height color scale.

\begin{figure*}[!tp]
\centering
\includegraphics[width=\textwidth]{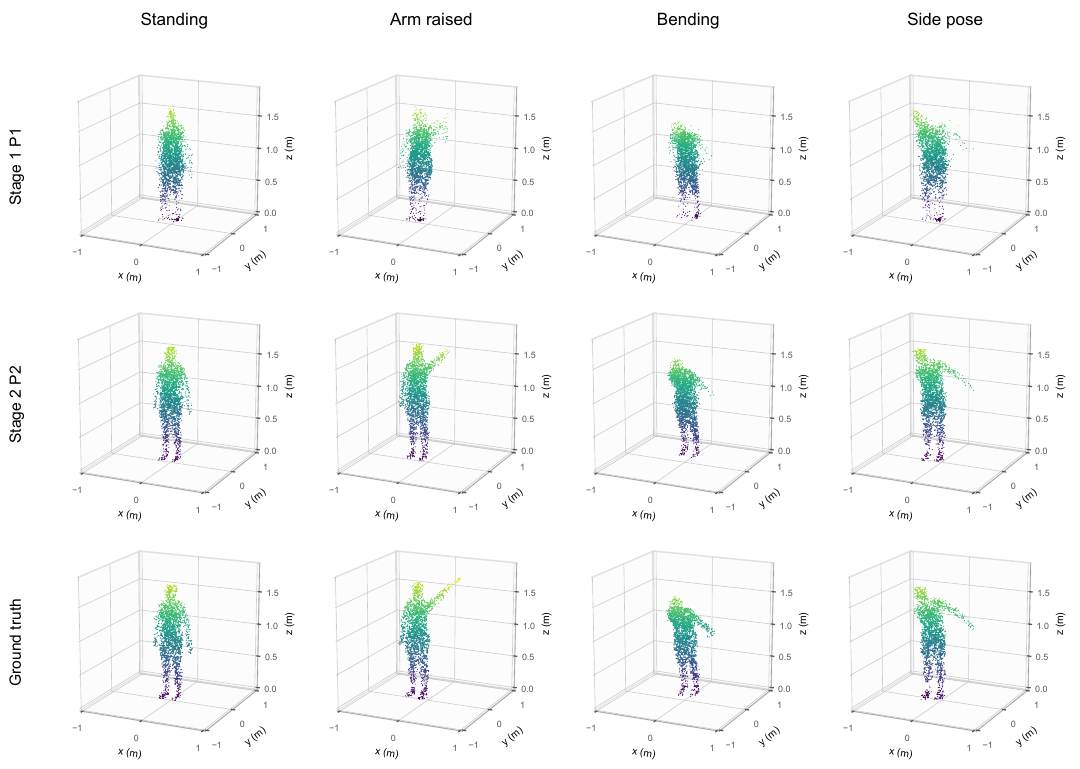}
\caption{Two-stage point-cloud evolution for four fixed env\_003 samples. Columns show standing, raised arm, bent torso, and side-facing poses; rows show stage P1, stage P2, and ground truth. All panels use the same coordinate limits, orthographic camera, point size, and height color scale}
\label{fig:deformation-stages}
\end{figure*}

\subsection{Cross-Environment Consistency and Geometry Objective}\label{sec:cross-environment-training}

During training, responses for the same human target in two training backgrounds form a pair. The two paths share model parameters but compute their own differences, background conditions, and physical maps. Let their predictions be $\widehat{\mathcal P}_0$ and $\widehat{\mathcal P}_1$, with common ground truth $\mathcal P^\star$. Define
\begin{equation}
\mathcal L_{\mathrm{cd}}=
\frac12\sum_{j=0}^{1}
\operatorname{CD}(\widehat{\mathcal P}_j,\mathcal P^\star),
\qquad
\mathcal L_{\mathrm{pair}}=
\operatorname{CD}(\widehat{\mathcal P}_0,\widehat{\mathcal P}_1),
\label{eq:geometry-loss}
\end{equation}
where CD is the symmetric squared Chamfer distance. The geometry objective is
\begin{equation}
\mathcal L_{\mathrm{geom}}=
\mathcal L_{\mathrm{cd}}+0.10\mathcal L_{\mathrm{pair}}.
\label{eq:geometry-objective}
\end{equation}
The ground-truth term anchors geometric accuracy, while the consistency term penalizes output drift between backgrounds.

The pairwise constraint is imposed on final point sets, not on the intermediate fields or adjoint maps, which legitimately differ across environments. A symmetric set distance does not require an indexwise match between predicted and true points. The relatively small consistency weight makes the term an auxiliary regularizer while retaining direct geometric supervision as the primary objective.

\section{Theoretical Analysis}\label{sec:theory}

The following results concern three distinct spaces: the complex measurement vector, the DOI contrast grid, and the output coordinates. They quantify, respectively, the residual after paired subtraction, the structure of the raw adjoint map, and the feasible set of the bounded decoder.

\subsection{Residual Bound for Background Subtraction}\label{sec:theory-difference}

Let $\mathbf y_e^{\mathrm{bg}}$ be the ideal background field and $\mathbf s_e(\chi)$ the target-induced scattering response. The target-present observation has complex error $\boldsymbol\eta^{\mathrm{obj}}$; the paired background observation has drift $\boldsymbol\delta_e$ and error $\boldsymbol\eta^{\mathrm{bg}}$:
\begin{equation}
\mathbf y_e^{\mathrm{obj}}=\mathbf y_e^{\mathrm{bg}}+
\mathbf s_e(\chi)+\boldsymbol\eta^{\mathrm{obj}},\qquad
\widetilde{\mathbf y}_e^{\mathrm{bg}}=\mathbf y_e^{\mathrm{bg}}+
\boldsymbol\delta_e+\boldsymbol\eta^{\mathrm{bg}}.
\label{eq:paired-observation-model}
\end{equation}

\begin{theorem}[Residual bound for paired subtraction]\label{thm:difference-bound}
Under Eq.~\eqref{eq:paired-observation-model}, let
\(\widetilde{\Delta\mathbf y}_e=\mathbf y_e^{\mathrm{obj}}-
\widetilde{\mathbf y}_e^{\mathrm{bg}}\), then
\begin{equation}
\left\|\widetilde{\Delta\mathbf y}_e-\mathbf s_e(\chi)\right\|_2
\le \|\boldsymbol\delta_e\|_2+
\|\boldsymbol\eta^{\mathrm{obj}}\|_2+
\|\boldsymbol\eta^{\mathrm{bg}}\|_2.
\label{eq:difference-bound}
\end{equation}
\end{theorem}

\begin{proof}
Subtracting the two observations in Eq.~\eqref{eq:paired-observation-model} gives
\begin{equation}
\widetilde{\Delta\mathbf y}_e-\mathbf s_e(\chi)
=-\boldsymbol\delta_e+\boldsymbol\eta^{\mathrm{obj}}-
\boldsymbol\eta^{\mathrm{bg}}.
\end{equation}
Applying the triangle inequality in the finite-dimensional complex measurement space gives Eq.~\eqref{eq:difference-bound}.
\end{proof}

The bound makes no assumption about the phases or correlation of the three perturbations. Shared background scattering cancels, whereas unshared drift and observation errors remain. Background conditioning can account for structured scene effects still present in the difference; the subtraction ablation tests the empirical value of this measurement preprocessing.

\subsection{Weighted Normal-Operator Interpretation}\label{sec:theory-normal}

Locally linearize the target contrast on the common DOI grid:
\begin{equation}
\Delta\mathbf y_e=\mathbf A_e\boldsymbol\chi+\mathbf r_e,
\label{eq:linearized-model}
\end{equation}
The residual $\mathbf r_e$ collects multiple scattering, discretization approximation, and remaining background mismatch. Temporarily fix the frequency-weight matrix $\mathbf W_e$. Before magnitude aggregation and spatial normalization, define the stacked adjoint $\mathbf g_e=\mathbf A_e^H\mathbf W_e\Delta\mathbf y_e$.

\begin{theorem}[Weighted normal-operator decomposition]\label{thm:normal-operator}
If Eq.~\eqref{eq:linearized-model} holds, then
\begin{equation}
\mathbf g_e=
\mathbf A_e^H\mathbf W_e\mathbf A_e
\boldsymbol\chi+\mathbf A_e^H\mathbf W_e\mathbf r_e.
\label{eq:normal-decomposition}
\end{equation}
Let $\mathbf H_e=\mathbf A_e^H\mathbf W_e\mathbf A_e$. If $\mathbf W_e$ is Hermitian positive semidefinite, then $\mathbf H_e$ is a positive-semidefinite weighted normal operator.
\end{theorem}

\begin{proof}
Substitute Eq.~\eqref{eq:linearized-model} into $\mathbf g_e=\mathbf A_e^H\mathbf W_e\Delta\mathbf y_e$ to obtain Eq.~\eqref{eq:normal-decomposition}. For any $\mathbf x$,
\begin{equation}
\mathbf x^H\mathbf H_e\mathbf x
=(\mathbf A_e\mathbf x)^H\mathbf W_e(\mathbf A_e\mathbf x)\ge0,
\end{equation}
thus $\mathbf H_e$ is positive semidefinite.
\end{proof}

The term $\mathbf H_e\boldsymbol\chi$ maps observable contrast back to the DOI; the finite aperture and frequency coverage determine its broadened peaks, sidelobes, and null space. The residual is backprojected by $\mathbf A_e^H\mathbf W_e\mathbf r_e$. The adjoint therefore supplies coarse support and location, while the learned encoder and pointwise decoder interpret ambiguous spatial evidence and recover a continuous surface.

\subsection{Feasible Set of the Bounded Point Cloud}\label{sec:theory-feasible}

Let each template point $\mathbf t_i$ belong to $\Omega$. At each stage, predict a bounded displacement according to Eq.~\eqref{eq:stage-update} and apply the box projection.

\begin{theorem}[Two-stage feasible set and displacement bound]\label{thm:deformation-bound}
Each final point satisfies
\begin{equation}
\widehat{\mathbf p}_i\in\Omega,
\qquad
\|\widehat{\mathbf p}_i-\mathbf t_i\|_\infty\le0.35\ \mathrm m.
\label{eq:deformation-bound}
\end{equation}
\end{theorem}

The hyperbolic tangent bounds the per-coordinate displacements by 0.25~m and 0.10~m. Coordinatewise box projection is nonexpansive:
\begin{equation}
\|\Pi_\Omega(\mathbf x)-\Pi_\Omega(\mathbf z)\|_\infty
\le\|\mathbf x-\mathbf z\|_\infty,
\label{eq:projection-nonexpansive}
\end{equation}
Applying nonexpansiveness at both stages and the triangle inequality gives the cumulative bound. The theorem guarantees that all 2048 points remain in the DOI and constrains extreme coordinate drift. Reconstruction accuracy still depends on template coverage, local evidence, and predicted displacement directions. A complete proof is provided in Online Resource 1 (Section S1).

\section{Experiments and Results}\label{sec:experiments}\label{sec:results}

\subsection{Experimental Setup}\label{sec:experimental-setup}

All data were generated by full-wave electromagnetic simulation in complex static environments. Each sample comprises target-present and background complex responses under the same acquisition geometry and a metric human point cloud. All scenes use the DOI in Eq.~\eqref{eq:doi}, four transmitters, twelve receiver locations, and frequencies of 300, 450, and 600~MHz. Only the static scattering structures and their layouts differ. env\_000 contains a floor; env\_001 adds walls on the positive $x$ and $y$ sides; and env\_002 contains a floor and three cuboid objects but no walls. env\_003 has a floor, walls on the negative $x$ and $y$ sides, and four objects. The independent env\_004 layout has a floor, walls on the positive $x$ and negative $y$ sides, and four objects. env\_005 has a floor, walls on the negative $x$ and positive $y$ sides, and four objects, with wall dimensions and object arrangements differing from the other environments. Figure~\ref{fig:measurement-configuration} illustrates env\_004.

The split is by environment (Table~\ref{tab:environment-split}). Each of env\_000--env\_002 contributes 640 training samples (1920 total); env\_003 has 570 validation samples; and env\_004 has 500 principal test samples. The new background env\_005 has 70 samples for extended pose evaluation, ten from each of seven pose groups already represented in training.

\begin{table}[!tb]
\centering
\caption{Environment structure and data split. Counts are the samples included in each set.}
\label{tab:environment-split}
\footnotesize
\setlength{\tabcolsep}{3pt}
\begin{tabularx}{\columnwidth}{@{}lXrl@{}}
\toprule
Environment & Static structure & Samples & Role \\
\midrule
env\_000 & Floor & 640 & Training \\
env\_001 & Floor, two walls & 640 & Training \\
env\_002 & Floor, three objects & 640 & Training \\
env\_003 & Floor, two walls, four objects & 570 & Validation \\
env\_004 & Floor, two walls, four objects & 500 & Main test \\
env\_005 & Floor, two walls, four objects & 70 & Extended test \\
\bottomrule
\end{tabularx}
\end{table}

\begin{table*}[!t]
\centering
\caption{Method comparison on the principal test environment env\_004. Trainable methods report means over three random seeds. The best mean for each metric is bold and shaded. CD-L2 is in m$^2$; Root-CD, centroid error, and H95 are in m.}
\label{tab:main-comparison}
\footnotesize
\setlength{\tabcolsep}{2pt}
\begin{tabular*}{\textwidth}{@{\extracolsep{\fill}}lrrrrrr@{}}
\toprule
 & \multicolumn{2}{c}{Overall geometry} & \multicolumn{2}{c}{Threshold coverage} & \multicolumn{2}{c}{Location and tail error} \\
\cmidrule(lr){2-3}\cmidrule(lr){4-5}\cmidrule(l){6-7}
Method & CD-L2\(\downarrow\) & Root-CD\(\downarrow\) & F@0.02\(\uparrow\) & F@0.05\(\uparrow\) & Centroid\(\downarrow\) & H95\(\downarrow\) \\
\midrule
\rowcolor{rowgray}Direct Born adjoint & \(0.544516\) & \(0.737076\) & \(0.040135\) & \(0.162922\) & \(0.264834\) & \(0.975135\) \\
Data-only TopNet & \(0.004825\) & \(0.069099\) & \(0.188703\) & \(0.755998\) & \(0.155409\) & \(0.092107\) \\
\textbf{R1A-PC} & \cellcolor{bestcell}\(\mathbf{0.001434}\) & \cellcolor{bestcell}\(\mathbf{0.037753}\) & \cellcolor{bestcell}\(\mathbf{0.449782}\) & \cellcolor{bestcell}\(\mathbf{0.963080}\) & \cellcolor{bestcell}\(\mathbf{0.025276}\) & \cellcolor{bestcell}\(\mathbf{0.045617}\) \\
\bottomrule
\end{tabular*}
\end{table*}

The baselines are a deterministic direct Born-adjoint reconstruction and data-only TopNet. The Born baseline chooses the 2048 voxel centres with the highest frequency-balanced adjoint scores. TopNet receives the same complex difference and generates points from a global latent code. Three separately retrained ablations remove physical guidance, background subtraction, or cross-environment consistency while retaining the remaining training and evaluation conditions.

For predicted cloud $\mathcal P$ and ground-truth cloud $\mathcal Q$, the principal error is the symmetric squared Chamfer distance
\begin{equation}
\operatorname{CD}(\mathcal P,\mathcal Q)=
\frac1{|\mathcal P|}\sum_{\mathbf p\in\mathcal P}
\min_{\mathbf q\in\mathcal Q}\|\mathbf p-\mathbf q\|_2^2+
\frac1{|\mathcal Q|}\sum_{\mathbf q\in\mathcal Q}
\min_{\mathbf p\in\mathcal P}\|\mathbf q-\mathbf p\|_2^2.
\label{eq:chamfer}
\end{equation}
We also report Root-CD, F-scores at 0.02 and 0.05~m, centroid error, and H95 (the 95th percentile of bidirectional nearest-neighbour distances). All metrics use the original metric coordinates. CD-L2 has units of m$^2$; Root-CD, centroid error, and H95 have units of m; the F-scores are dimensionless. Trainable methods report the mean and population standard deviation across three seeds. Target-level comparisons on env\_004 use paired bootstrap intervals and two-sided paired Wilcoxon tests, with Holm correction for prespecified comparisons. Detailed optimization and timing settings are provided in Online Resource 1 (Section S3).

\subsection{Main Test Results}\label{sec:main-results}

Table~\ref{tab:main-comparison} reports the env\_004 results. R1A-PC achieves CD-L2 of \(0.001434\pm0.000087\ \mathrm m^2\), F@0.05 of \(0.963080\pm0.003883\), and H95 of \(0.045617\pm0.001183\ \mathrm m\). Relative to TopNet, CD-L2 falls by 70.28\%, F@0.05 rises by 0.207082, and H95 falls by 50.47\%. The Holm-adjusted $p$ values reported for all three comparisons are \(2.66454\times10^{-14}\). These measures cover mean fit, threshold coverage, and tail error.

F@0.02 increases from 0.188703 for TopNet to 0.449782 for R1A-PC. Root-CD and centroid error fall to 0.037753 and 0.025276~m, respectively. Paired subtraction attenuates dominant scene components, while the adjoint map supplies a spatial anchor for shape refinement.

Figure~\ref{fig:qualitative-main} compares four fixed validation examples from env\_003. Direct Born reconstruction roughly locates the body but remains affected by sidelobes and spatial blur. TopNet generates a coherent global shape but can err in absolute location or limb direction. R1A-PC more closely follows the raised-arm, bent, and side-facing reference shapes, consistent with coarse adjoint localization followed by learned surface refinement.

\begin{figure*}[!tp]
\centering
\includegraphics[width=0.82\textwidth]{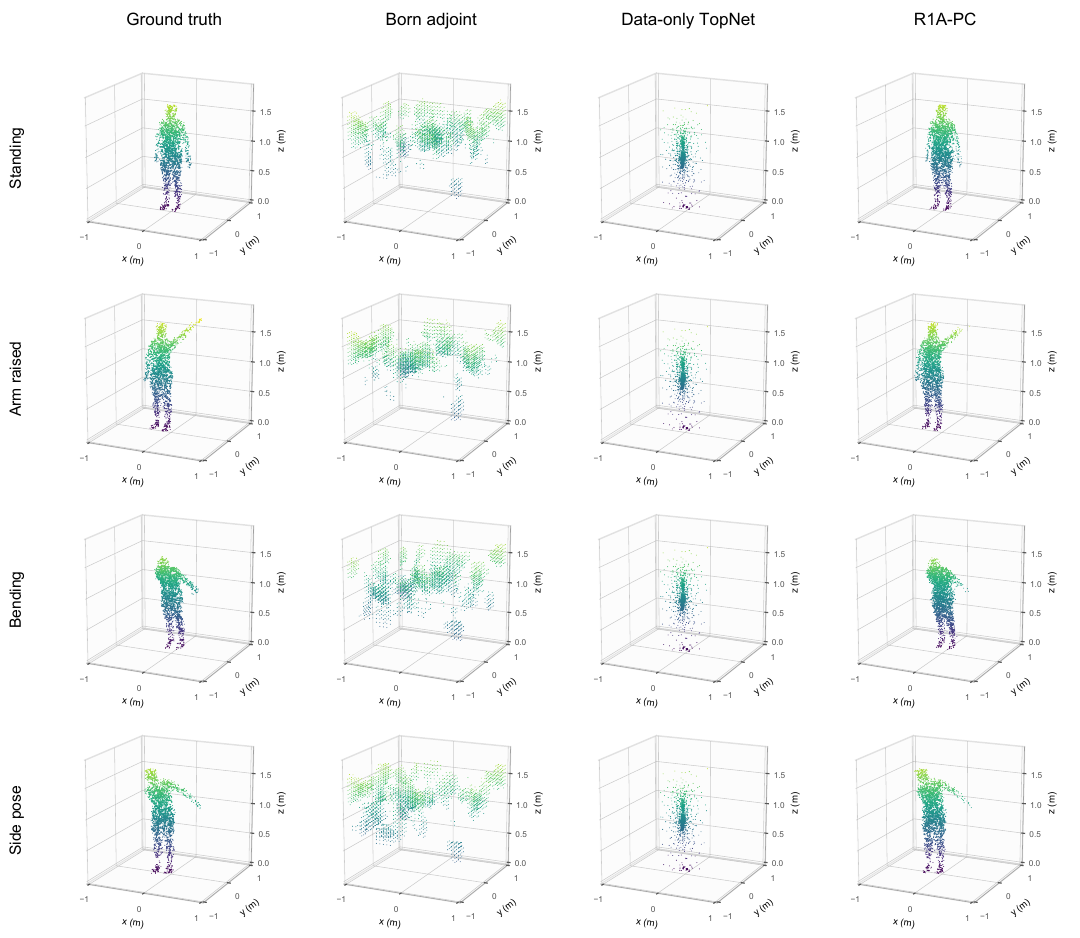}
\caption{Qualitative comparison for four env\_003 poses. Rows show standing, raised arm, bent torso, and side-facing poses; columns show ground truth, Born adjoint, data-only TopNet, and R1A-PC. Colour encodes height $z$. All panels use metric coordinates, common DOI limits, an orthographic camera, and the same point size}
\label{fig:qualitative-main}
\end{figure*}

\begin{figure*}[!tp]
\centering
\includegraphics[width=0.94\textwidth]{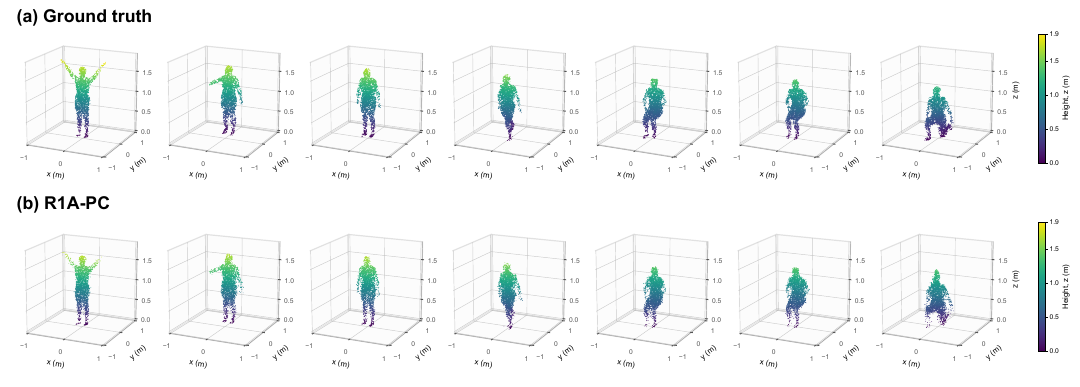}
\caption{Seven extended poses in env\_005. The upper row shows ground truth and the lower row R1A-PC. Columns show both arms raised, one arm extended forward, slight forward lean, standing lunge, natural sitting, forward-leaning sitting, and one-knee kneeling}
\label{fig:extended-seven-actions}
\end{figure*}

Figure~\ref{fig:extended-seven-actions} illustrates the seven extended poses in env\_005. Figure~\ref{fig:target-distribution} displays target-level error distributions after averaging each target over the three seeds. The empirical distribution for R1A-PC lies to the left of TopNet's. The R1A-PC minus TopNet CD-L2 differences are negative for all 500 test targets, so the mean advantage in Table~\ref{tab:main-comparison} extends across the principal test set.

\begin{figure}[!tb]
\centering
\ifdim\columnwidth<\textwidth
  \includegraphics[width=0.86\columnwidth]{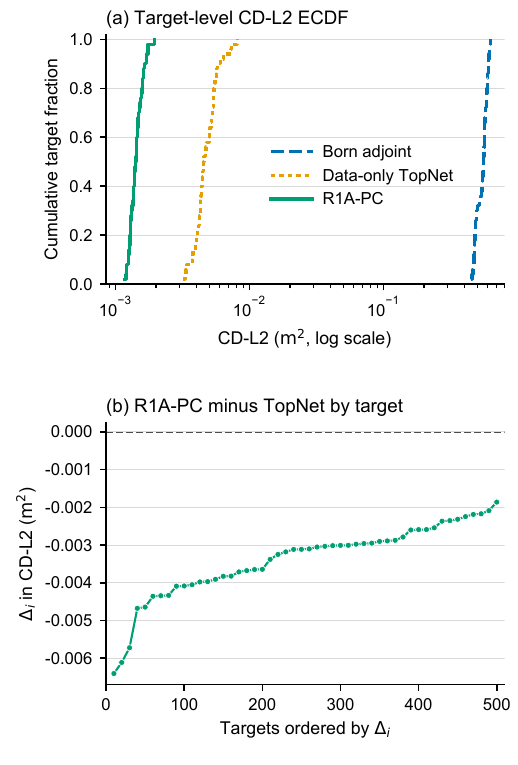}
\else
  \includegraphics[width=0.58\textwidth]{Fig6.pdf}
\fi
\caption{Reconstruction-error distributions for 500 test targets in env\_004. (a) Target-level empirical cumulative distributions of CD-L2 for the Born adjoint, data-only TopNet, and R1A-PC. (b) Target-level differences in CD-L2 between R1A-PC and TopNet; the zero line denotes equal errors}
\label{fig:target-distribution}
\end{figure}

\subsection{Component Ablations}\label{sec:ablation-results}

As Table~\ref{tab:ablation} shows, removing physical guidance or background subtraction increases CD-L2 by 242.06\% or 241.18\%, respectively. The former provides location-dependent spatial evidence, while the latter suppresses the shared static response at the measurement level.

\begin{table}[!tb]
\centering
\caption{Component ablations. Values are means of three independent trainings; relative CD-L2 changes use complete R1A-PC as the reference. CD-L2 is in m$^2$ and H95 in m.}
\label{tab:ablation}
\footnotesize
\setlength{\tabcolsep}{1.5pt}
\begin{tabular*}{\columnwidth}{@{\extracolsep{\fill}}lrrrr@{}}
\toprule
 & \multicolumn{3}{c}{Reconstruction quality} & \multicolumn{1}{c}{Relative change} \\
\cmidrule(lr){2-4}\cmidrule(l){5-5}
Method & CD-L2\(\downarrow\) & F@0.05\(\uparrow\) & H95\(\downarrow\) & \(\Delta\)CD-L2 \\
\midrule
\rowcolor{rowgray}Without physical guidance & \(0.004907\) & \(0.763553\) & \(0.092685\) & \(242.06\) \\
Without background subtraction & \(0.004894\) & \(0.763412\) & \(0.092544\) & \(241.18\) \\
\rowcolor{rowgray}Without cross-environment consistency & \(0.001459\) & \(0.962987\) & \(0.045780\) & \(1.7\) \\
\textbf{R1A-PC} & \cellcolor{bestcell}\(\mathbf{0.001434}\) & \cellcolor{bestcell}\(\mathbf{0.963080}\) & \cellcolor{bestcell}\(\mathbf{0.045617}\) & \(0.0\) \\
\bottomrule
\end{tabular*}
\end{table}

Without cross-environment consistency, CD-L2 is \(0.001459\ \mathrm m^2\), a 1.70\% increase with adjusted \(p=0.0012979\). The corresponding adjusted $p$ values for F@0.05 and H95 are 0.908598 and 0.408952. This output-level regularizer provides a small gain in mean geometry error but has limited effect on the other two metrics.

\subsection{Cross-Environment Results}\label{sec:cross-environment-results}

Table~\ref{tab:cross-environment} reports environment-level accuracy on env\_002, env\_003, and env\_004.

\begin{table}[!tb]
\centering
\caption{R1A-PC reconstruction results on three background levels; means of three independent trainings. CD-L2 is in m$^2$ and centroid error and H95 are in m.}
\label{tab:cross-environment}
\footnotesize
\setlength{\tabcolsep}{1.5pt}
\begin{tabular*}{\columnwidth}{@{\extracolsep{\fill}}lrrrr@{}}
\toprule
 & \multicolumn{2}{c}{Overall accuracy} & \multicolumn{2}{c}{Location and tail error} \\
\cmidrule(lr){2-3}\cmidrule(l){4-5}
Environment & CD-L2\(\downarrow\) & F@0.05\(\uparrow\) & Centroid\(\downarrow\) & H95\(\downarrow\) \\
\midrule
\rowcolor{rowgray}env\_002 & \(0.001210\) & \(0.976085\) & \(0.023693\) & \(0.041576\) \\
env\_003 & \(0.001592\) & \(0.950150\) & \(0.029864\) & \(0.049623\) \\
\rowcolor{rowgray}env\_004 & \(0.001434\) & \(0.963080\) & \(0.025276\) & \(0.045617\) \\
\bottomrule
\end{tabular*}
\end{table}

Errors vary with the background propagation structure but remain of a similar order of magnitude. Background encoding conditions the difference measurements on the scene, while the adjoint map preserves spatial localization.

\subsection{Extended Pose Results}\label{sec:multi-pose-results}

Results for the seven poses in env\_005 appear in Table~\ref{tab:multi-pose}, with examples in Fig.~\ref{fig:extended-seven-actions}. R1A-PC obtains CD-L2 of \(0.004845\pm0.000270\ \mathrm m^2\), F@0.05 of \(0.863498\pm0.012397\), and H95 of \(0.092827\pm0.008801\ \mathrm m\).

\begin{table}[!tb]
\centering
\caption{Mean performance over seven pose groups in the new environment env\_005. Trainable methods report means across three random seeds. CD-L2 is in m$^2$ and H95 in m.}
\label{tab:multi-pose}
\footnotesize
\begin{tabular*}{\columnwidth}{@{\extracolsep{\fill}}lrrr@{}}
\toprule
 & \multicolumn{3}{c}{New-environment pose reconstruction} \\
\cmidrule(l){2-4}
Method & CD-L2\(\downarrow\) & F@0.05\(\uparrow\) & H95\(\downarrow\) \\
\midrule
\rowcolor{rowgray}Direct Born adjoint & \(0.608902\) & \(0.143379\) & \(1.018\) \\
Data-only TopNet & \(0.008265\) & \(0.621074\) & \(0.114149\) \\
\textbf{R1A-PC} & \cellcolor{bestcell}\(\mathbf{0.004845}\) & \cellcolor{bestcell}\(\mathbf{0.863498}\) & \cellcolor{bestcell}\(\mathbf{0.092827}\) \\
\bottomrule
\end{tabular*}
\end{table}

Compared with TopNet, R1A-PC reduces CD-L2 and H95 and raises F@0.05 in this new background. Each group has ten samples; the figure shows one example per group, whereas the table summarizes all 70 samples.

\subsection{Computational Efficiency}\label{sec:efficiency-results}

Table~\ref{tab:efficiency} compares parameters, inference time, and peak memory on the same hardware. R1A-PC has 1,114,023 trainable parameters, takes \(6.460\pm0.415\)~ms per sample on average, and uses 68.3~MiB peak memory. TopNet takes 2.164~ms and uses 16.8~MiB. The additional work is mainly the three-dimensional physical-feature encoding and multiscale querying at 2048 current point positions.

\begin{table}[!tb]
\centering
\caption{Computational efficiency on an NVIDIA GeForce RTX 4060 Laptop GPU. Times are mean milliseconds per sample after input construction; peak memory is in MiB.}
\label{tab:efficiency}
\footnotesize
\setlength{\tabcolsep}{1.5pt}
\begin{tabular*}{\columnwidth}{@{\extracolsep{\fill}}lrrrr@{}}
\toprule
 & \multicolumn{2}{c}{Model size} & \multicolumn{2}{c}{Inference cost} \\
\cmidrule(lr){2-3}\cmidrule(l){4-5}
Method & Parameters & Output points & Time\(\downarrow\) & Memory\(\downarrow\) \\
\midrule
\rowcolor{rowgray}Direct Born adjoint & \(0\) & \(2\,048\) & \(29.803\) & \(140.3\) \\
Data-only TopNet & \(1\,029\,645\) & \(2\,048\) & \(2.164\) & \(16.8\) \\
\rowcolor{rowgray}\textbf{R1A-PC} & \(1\,114\,023\) & \(2\,048\) & \(6.460\) & \(68.3\) \\
\bottomrule
\end{tabular*}
\end{table}

Direct Born reconstruction takes 29.803~ms and uses 140.3~MiB; explicit voxel scoring, sorting, and point selection contribute to the cost. All times exclude construction of the model inputs.

\section{Conclusion}\label{sec:conclusion}

In complex static environments, strong multipath responses couple weak human perturbations to the background, while limited acquisition channels restrict the spatial resolution of three-dimensional inversion. R1A-PC combines paired complex background subtraction, a frequency-balanced conjugate Born adjoint, and multiscale physical features queried at the current locations of deforming points. Complex measurements, a background condition, and local spatial evidence jointly drive two bounded stages from a shared template to a 2048-point human cloud.

On the background-disjoint env\_004 main test, CD-L2, F@0.05, and H95 are \(0.001434\ \mathrm m^2\), 0.963080, and \(0.045617\ \mathrm m\), respectively. CD-L2 is 70.28\% below TopNet. Removing physical guidance or background subtraction increases CD-L2 by 242.06\% or 241.18\%. In the new-background pose evaluation, CD-L2 is \(0.004845\ \mathrm m^2\) and F@0.05 is 0.863498. Cross-environment consistency acts mainly as an auxiliary regularizer.

These findings are based on full-wave simulation under one fixed acquisition configuration: four transmitters, twelve receiver positions, three frequencies, and predetermined DOI and material settings. Training, validation, and principal testing are separated by background. Real-world deployment will require calibration for actual array positions, material properties, and a practicable paired-background reference.

\FloatBarrier
\backmatter
\section*{Statements and Declarations}
\bmhead{Funding}
The authors received no funding for this study.
\bmhead{Competing interests}
The authors declare that they have no competing interests.
\bmhead{Data availability}
The simulation data supporting the findings of this study are not publicly available but may be obtained from the corresponding author upon reasonable request.
\bibliographystyle{sn-basic}
\bibliography{references}
\end{document}